\documentclass[pdflatex,sn-mathphys-num]{sn-jnl}

\usepackage{amsmath,amssymb,amsfonts,mathtools}
\usepackage{amsthm}
\usepackage{booktabs}
\usepackage{url}

\theoremstyle{thmstyleone}
\newtheorem{theorem}{Theorem}[section]
\newtheorem{lemma}[theorem]{Lemma}
\newtheorem{corollary}[theorem]{Corollary}
\newtheorem{proposition}[theorem]{Proposition}

\theoremstyle{thmstyletwo}
\newtheorem{remark}[theorem]{Remark}

\theoremstyle{thmstylethree}
\newtheorem{definition}{Definition}[section]
\newtheorem{assumption}[definition]{Assumption}

\newcommand{\relu}{\operatorname{ReLU}}
\newcommand{\E}{\mathbb{E}}
\newcommand{\R}{\mathbb{R}}
\newcommand{\norm}[1]{\left\lVert #1 \right\rVert}
\newcommand{\abs}[1]{\left\lvert #1 \right\rvert}

\newcommand{\SigX}{\Sigma_X}

\newcommand{\SubBar}{\mathfrak b}

\begin{document}

\title[Approximation Rates and Barrier Decay]{From Approximation Rates to Loss-Landscape Barrier Decay in Shallow ReLU Networks}

\author*[1]{\fnm{Saveliy} \sur{Baturin}}\email{saveliy.xo@gmail.com}

\affil*[1]{\orgname{Independent Researcher}, \orgaddress{\country{Russia}}}

\abstract{We study pathwise connectivity of sublevel sets for one-hidden-layer ReLU networks with constrained first-layer weights and an $\ell_1$ penalty on the output layer. The data term is assumed convex and globally Lipschitz in the scalar logit. We first give a finite-width construction that connects any two points of a common sublevel through a path controlled by a loss-consistent compression functional and a first-order perturbation term. The proof replaces the quadratic perturbation estimate in the Freeman--Bruna mechanism by a direct Lipschitz bound. Positive homogeneity is then used in a direction that is compatible with the penalty: every active atom is moved monotonically from the unit ball to the unit sphere while its output coefficient is reduced. Sphere covering and cluster merging consequently give $O(m^{-1/(n-1)})$ fixed-level thickening for $n\ge2$, while the one-dimensional two-ray dictionary gives exact connectivity for every $m\ge4$. We also prove internally that the regularized approximation values satisfy $e(l)-e_\infty=O(l^{-1/2})$. More generally, a rate $O(l^{-s})$ transfers to a near-optimal barrier rate $O(m^{-s/((n-1)s+1)})$; under the standing assumptions, this yields the explicit rate $O(m^{-1/(n+1)})$. A theorem-aligned finite-distribution experiment complements the analysis. The primary Huber run yields a maximal best certified upper gap $1.66\times10^{-5}$ over 720 recorded pairs at widths $m\ge16$; a matched binary-cross-entropy rerun and a 720-endpoint dense-representation stress test probe loss robustness and the active cluster-merging mechanism.}



\maketitle

\section{Introduction}\label{sec:introduction}

The objective of this paper is a quantitative question about a specific nonconvex model: how much must the regularized risk increase when two parameter points in the same sublevel set are joined by a continuous path?  The answer is expressed through an approximation quantity of the underlying shallow-network class.

Freeman and Bruna \cite{FreemanBruna2017} proved an asymptotic connectivity mechanism for half-rectified one-hidden-layer networks with quadratic loss. Their argument combines output-layer convexity, sparse compression of the hidden representation, and the occurrence of clusters of first-layer directions as width grows. The quadratic structure enters when the extended convex path is replaced by a path in the original parameter space. We show that this step has a direct first-order substitute whenever the loss is convex and globally Lipschitz in the logit.

The parameter constraint needs particular care. Positive homogeneity gives
\[
 a\relu(w^\top x)=\frac{a}{c}\relu((cw)^\top x),\qquad c>0,
\]
but this rescaling changes an $\ell_1$ penalty imposed only on $a$. Consequently, normalizing the first layer is not a harmless ``without loss of generality'' operation. We formulate the model on a product of closed Euclidean unit balls. This is a genuine architectural constraint, and its convexity ensures that the first-layer interpolations used below remain admissible.

The main contributions are the following.
\begin{enumerate}
\item We prove an attained, finite-width barrier bound for losses that are convex and Lipschitz in the logit. The proof includes an explicit perturbation estimate and an explicit relocation argument for inactive neurons.
\item We prove a uniform statement over all endpoint pairs in a fixed sublevel. A monotone sphericalization path reduces the relevant metric dimension from $n$ to $n-1$, giving $O(m^{-1/(n-1)})$ thickening for $n\ge2$ and exact connectivity for every $m\ge4$ when $n=1$.
\item We connect approximation theory to landscape topology. A Maurey-type sampling argument gives the model-internal rate $e(l)-e_\infty=O(l^{-1/2})$; under the standing assumptions, near-optimal barriers are therefore $O(m^{-1/(n+1)})$ for $n\ge2$. We also state the sharper transfer produced by any available rate $O(l^{-s})$.
\item We give a theorem-aligned computational diagnostic on finite seeded distributions. It verifies the parameter and compression invariants and compares direct interpolation, certified Dynamic String Sampling (DSS), and the proof-inspired path on 900 trained-endpoint pairs, with disjoint pairing inside each replicate--width--level group. We repeat the complete design with binary cross-entropy in logits and separately stress nontrivial merging on dense endpoint representations.
\end{enumerate}

The asymptotic results concern the evolution of the approximation class with width; they are not convergence theorems for gradient descent or claims about training-time dynamics. Apart from the explicit one-dimensional result at $m\ge4$, the paper does not assert exact finite-width connectivity or make a statement about arbitrary deep networks.

The remainder is organized as follows. Section~\ref{sec:related} positions the result within connectivity, mode-connectivity, and convex-reformulation work. Section~\ref{sec:setting} fixes the model and the barrier quantities. Section~\ref{sec:finite} proves the finite-width path theorem. Section~\ref{sec:rates} derives the fixed-level and approximation-driven rates. Section~\ref{sec:computational} reports the theorem-aligned computational diagnostic, and Section~\ref{sec:discussion} discusses scope and limitations.

\section{Related Work}\label{sec:related}

\subsection{Sublevel-set topology}
The closest point of departure is Freeman and Bruna \cite{FreemanBruna2017}. For a regularized half-rectified shallow model with quadratic loss, they bound the excess energy of a path through a sparse-approximation functional and a perturbation of hidden directions, and then obtain asymptotic connectivity from a covering argument. Our finite-width construction retains that compression--relocation architecture, but replaces the loss-specific quadratic perturbation step by the first-order Lipschitz bridge in Lemma~\ref{lem:bridge}. We also treat the closed first-layer ball as an architectural constraint and state uniform fixed-level and approximation-transfer rates.

Venturi, Bandeira, and Bruna \cite{Venturi2019} characterize spurious valleys of one-hidden-layer objectives through intrinsic dimension, whereas Nguyen \cite{Nguyen2019} proves exact connectedness of every empirical-loss sublevel for sufficiently overparameterized piecewise-linear deep networks under sample-size-dependent width conditions. These are exact or qualitative topology results under assumptions different from ours. Simsek et al.\ \cite{Simsek2021} instead describe connected manifolds of global minima generated by overparameterization and permutation symmetry. Our object is neither only the minimizer set nor a sample-size width threshold: it is the worst excess along paths between arbitrary points of a prescribed regularized-risk sublevel.

Nurisso, Leroy, and Vaccarino \cite{Nurisso2024} identify disconnected components of the quadric invariant sets imposed on shallow ReLU gradient flow by homogeneity. Their subsequent DAG-architecture analysis with Petri \cite{Nurisso2026} characterizes connectedness and singularities of the corresponding training-invariant learning spaces. This is complementary to our result: they study reachability constrained by gradient-flow conservation laws, whereas our paths range over the full constrained parameter sublevel and are not asserted to be optimization trajectories. Wu, Simsek, and Ged \cite{Wu2025} analyze directional stationary points, saddle escape, and width embeddings for shallow ReLU-like networks with empirical squared loss; our theorem instead concerns arbitrary sublevel endpoints and convex globally Lipschitz losses.

\subsection{Approximate connectivity and trained solutions}
Kuditipudi et al.~\cite{Kuditipudi2019} explain low-cost piecewise-linear connections from dropout and noise stability of trained networks. Shevchenko and Mondelli \cite{ShevchenkoMondelli2020} make the width dependence explicit for SGD solutions, deriving vanishing barriers from increasing dropout stability. Ferbach et al.~\cite{Ferbach2024} give a different probabilistic explanation: Wasserstein convergence of neuron empirical measures yields linear connectivity, after alignment, for sufficiently wide networks trained by SGD. These results condition on properties or distributions of trained solutions. Here no training algorithm appears in the theorem and the endpoints are arbitrary elements of a stated sublevel; the price is the narrower constrained shallow architecture and a generally nonlinear constructed path.

\subsection{Convex reformulations}
Convex reformulations expose additional structure in regularized two-layer ReLU training. Mishkin and Pilanci \cite{MishkinPilanci2023}, for example, characterized optimal sets and studied continuous regularization paths. Such results address optimal or stationary sets through a finite-data convex representation. The present analysis instead gives a direct path bound for expected risk, arbitrary sublevel endpoints, and a constrained dictionary of ReLU atoms.

\subsection{Empirical mode connectivity}
Low-loss paths between independently trained networks were first constructed at scale by Draxler et al.\ and Garipov et al.\ \cite{Draxler2018,Garipov2018}. A later line of work asks whether much simpler linear paths emerge after quotienting neuron-permutation symmetries. Entezari et al.\ \cite{Entezari2022} formulated and empirically tested the permutation-invariance conjecture; Ainsworth et al.\ \cite{Ainsworth2023} developed weight-matching algorithms for model merging; and Jordan et al.\ \cite{Jordan2023} showed that activation-variance collapse can remain after alignment and proposed a renormalization repair. More recently, Zhao et al.\ \cite{Zhao2025} related mode connectivity to continuous parameter symmetries and derived explicit symmetry-induced connecting curves. This literature concerns trained modes, alignment, and often deep architectures. It motivates the direct-interpolation baseline in Section~\ref{sec:computational}, but does not imply the uniform sublevel theorem proved here; conversely, our theorem neither identifies permutation alignments nor claims linear connectivity.

\subsection{Approximation theory}
Quantitative approximation by neural networks has a long history; representative entry points include Barron's sampling bounds \cite{Barron1993} and Pinkus's survey of multilayer perceptron approximation \cite{Pinkus1999}. Recent work develops optimal rates for shallow ReLU variation spaces \cite{YangZhou2025,Siegel2025}; convex Lipschitz losses also appear in contemporary constructive-approximation analyses of kernel approximations \cite{Lian2026}. Our approximation statement has a different purpose. Rather than posit a target-class norm and estimate its best function error, Lemma~\ref{lem:maurey} sparsifies near-minimizers of the exact regularized objective. This gives a universal $O(l^{-1/2})$ bound for $e(l)-e_\infty$ under the standing hypotheses and then transfers that bound to landscape topology. Faster target-specific estimates, when available for the same bias-free dictionary, distribution, loss, and penalty, can still be inserted into Corollary~\ref{cor:transfer}.

\section{Model and Quantities}\label{sec:setting}

Let
\[
B_2^n=\{w\in\R^n:\norm{w}_2\le 1\},
\qquad
\mathcal W_m=(B_2^n)^m\times\R^m.
\]
For $\vartheta=(W,\theta)\in\mathcal W_m$, where the rows of $W$ are $w_1,\ldots,w_m$, define
\begin{equation}
\Phi_m(x;W,\theta)
=\sum_{i=1}^m\theta_i\relu(w_i^\top x).
\label{eq:model}
\end{equation}
The first-layer ball constraint is part of the model. In particular, the objective below would be ill posed under unrestricted positive rescaling of $W$ with a penalty only on $\theta$.

Let $(X,Y)$ have distribution $P$, and let
\begin{equation}
R_m(W,\theta)
=\E\!\left[\mathcal L\!\left(Y,\Phi_m(X;W,\theta)\right)\right],
\qquad
F_m(W,\theta)=R_m(W,\theta)+\kappa\norm{\theta}_1,
\label{eq:objective}
\end{equation}
where $\kappa>0$. Write $\SigX=\E[XX^\top]$ and
\[
D_X=\sqrt{\norm{\SigX}_{\mathrm{op}}}.
\]

\begin{assumption}[Standing hypotheses]\label{ass:standing}
The second moment of $X$ is finite, and $\E[\abs{\mathcal L(Y,0)}]<\infty$. For $P_Y$-almost every $y$, the map $u\mapsto\mathcal L(y,u)$ is convex and $L$-Lipschitz, with the same finite constant $L$. There is a finite constant $\underline R$ such that
\[
R_m(W,\theta)\ge \underline R
\]
for every width $m$ and every $(W,\theta)\in\mathcal W_m$.
\end{assumption}

For example, binary cross-entropy with a scalar logit $u$ and label $y\in\{0,1\}$ is
\begin{equation}
\mathcal L_{\rm BCE}(y,u)=\log(1+e^u)-yu.
\label{eq:bce-logit}
\end{equation}
It is convex and globally $1$-Lipschitz in the logit because
$\partial_u\mathcal L_{\rm BCE}(y,u)=\sigma(u)-y\in[-1,1]$.
Thus no bounded-logit assumption is needed. Quadratic loss, by contrast, is not globally Lipschitz.

\begin{definition}[Sublevel and pairwise barrier]
For $\lambda\in\R$, set
\[
\Omega_m(\lambda)=\{\vartheta\in\mathcal W_m:F_m(\vartheta)\le\lambda\}.
\]
For $\vartheta^A,\vartheta^B\in\mathcal W_m$, define
\begin{equation}
\operatorname{bar}_m(\vartheta^A,\vartheta^B)
=\inf_{\gamma}\left(
\max_{t\in[0,1]}F_m(\gamma(t))
-\max\{F_m(\vartheta^A),F_m(\vartheta^B)\}
\right),
\label{eq:pair-barrier}
\end{equation}
where the infimum is over continuous paths in $\mathcal W_m$ with the prescribed endpoints.
\end{definition}

For a uniform sublevel statement it is convenient to measure thickening above the specified level, rather than above the individual endpoint values.

\begin{definition}[Sublevel thickening]\label{def:thickening}
Define
\begin{equation}
\SubBar_m(\lambda)
=\sup_{\vartheta^A,\vartheta^B\in\Omega_m(\lambda)}
\inf_{\gamma}
\left(\max_{t\in[0,1]}F_m(\gamma(t))-\lambda\right)_+.
\label{eq:thickening}
\end{equation}
If the sublevel is empty, the supremum is set to zero.
\end{definition}

The approximation values used below are
\begin{equation}
e(l)
=\min_{(U,\xi)\in\mathcal W_l}F_l(U,\xi),
\qquad
e_\infty=\lim_{l\to\infty}e(l).
\label{eq:e-l}
\end{equation}
The limit exists because $e(l)$ is nonincreasing: an $l$-neuron model embeds into width $l+1$ by adding a zero output coefficient.

For $W\in(B_2^n)^m$, an active-neuron budget $q\le m$, and $r\ge0$, define the loss-consistent robust compression value
\begin{equation}
\delta_W^{\mathcal L}(q,r;m)
=\min_{\substack{\gamma\in\R^m,\ \norm{\gamma}_0\le q\\
\widehat W\in(B_2^n)^m,\ \max_i\norm{\hat w_i-w_i}_2\le r}}
F_m(\widehat W,\gamma).
\label{eq:delta}
\end{equation}
Unlike an angular definition on the sphere, \eqref{eq:delta} is compatible with convex interpolation inside the parameter domain.

\begin{lemma}[Attainment and output-layer control]\label{lem:attainment}
Under Assumption~\ref{ass:standing}, the minima in \eqref{eq:e-l} and \eqref{eq:delta} are attained. Moreover,
\begin{equation}
F_m(W,\theta)\le s
\quad\Longrightarrow\quad
\norm{\theta}_1\le\frac{s-\underline R}{\kappa}.
\label{eq:l1-control}
\end{equation}
\end{lemma}

\begin{proof}
The implication \eqref{eq:l1-control} follows from
\[
s\ge R_m(W,\theta)+\kappa\norm{\theta}_1
\ge\underline R+\kappa\norm{\theta}_1.
\]
For attainment, take a minimizing sequence. Its objective values are bounded above, so \eqref{eq:l1-control} bounds its output coefficients. The first-layer variables range over a compact set, and the sparsity constraint in \eqref{eq:delta} is a finite union of closed coordinate subspaces. The additional rowwise distance constraint in \eqref{eq:delta} is closed. Hence a convergent subsequence $(W_j,\theta_j)\to(W,\theta)$ exists. On a set where the output coefficients are bounded, the ReLU inequality
\[
\E\abs{\relu(a^\top X)-\relu(b^\top X)}
\le D_X\norm{a-b}_2
\label{eq:relu-lipschitz}
\]
gives the explicit joint bound
\[
\E\abs{\Phi_m(X;W_j,\theta_j)-\Phi_m(X;W,\theta)}
\le D_X\norm{\theta_j-\theta}_1
   +D_X\norm{\theta}_1\max_i\norm{w_{j,i}-w_i}_2,
\]
because every first-layer row has norm at most one. The right-hand side tends to zero. The Lipschitz property of $\mathcal L$ therefore implies continuity of the expected risk, while the $\ell_1$ term is continuous in finite dimension. The subsequential limit attains the infimum.
\end{proof}

\section{A Finite-Width Path Bound}\label{sec:finite}

We first isolate the perturbation step that replaces the quadratic estimate.

\begin{lemma}[Lipschitz bridge]\label{lem:bridge}
Let $(W,\beta),(\widehat W,\gamma)\in\mathcal W_m$ satisfy
\[
\max_i\norm{w_i-\hat w_i}_2\le r,
\qquad
F_m(W,\beta)\le\Lambda,
\qquad
F_m(\widehat W,\gamma)\le\Lambda.
\]
Then the path
\[
W_t=(1-t)W+t\widehat W,
\qquad
\theta_t=(1-t)\beta+t\gamma
\]
lies in $\mathcal W_m$ and satisfies
\begin{equation}
F_m(W_t,\theta_t)
\le
\Lambda+\frac{LD_X}{\kappa}r(\Lambda-\underline R),
\qquad 0\le t\le1.
\label{eq:bridge-bound}
\end{equation}
\end{lemma}

\begin{proof}
The unit ball is convex, so every row of $W_t$ is admissible. Introduce the extended convex logit
\[
g_t(X)=(1-t)\Phi_m(X;W,\beta)
+t\Phi_m(X;\widehat W,\gamma).
\]
Convexity of the loss and of the $\ell_1$ norm gives
\begin{equation}
\E[\mathcal L(Y,g_t(X))]
+\kappa\big((1-t)\norm{\beta}_1+t\norm{\gamma}_1\big)
\le\Lambda.
\label{eq:extended-convex}
\end{equation}
For $w_{t,i}=(1-t)w_i+t\hat w_i$, the difference between the genuine and extended logits is
\begin{align*}
\Phi_m(X;W_t,\theta_t)-g_t(X)
={}&\sum_i(1-t)\beta_i
\bigl(\relu(w_{t,i}^\top X)-\relu(w_i^\top X)\bigr)\\
&+\sum_i t\gamma_i
\bigl(\relu(w_{t,i}^\top X)-\relu(\hat w_i^\top X)\bigr).
\end{align*}
Using \eqref{eq:relu-lipschitz}, $\norm{w_{t,i}-w_i}\le r$, and $\norm{w_{t,i}-\hat w_i}\le r$, its expected absolute value is at most
\[
rD_X\big((1-t)\norm{\beta}_1+t\norm{\gamma}_1\big).
\]
Lemma~\ref{lem:attainment} bounds both $\ell_1$ norms by $(\Lambda-\underline R)/\kappa$. Combining the Lipschitz estimate for $\mathcal L$ with \eqref{eq:extended-convex}, and using $\norm{\theta_t}_1\le(1-t)\norm{\beta}_1+t\norm{\gamma}_1$, proves \eqref{eq:bridge-bound}.
\end{proof}

\begin{theorem}[Finite-width connectivity bound]\label{thm:finite}
Assume Assumption~\ref{ass:standing}. Let $\vartheta^A=(W_A,\theta_A)$ and $\vartheta^B=(W_B,\theta_B)$ belong to $\Omega_m(\lambda)$. For every integer $1\le l\le m$ and every $r\ge0$, there is a continuous path $\gamma_{l,r}$ in $\mathcal W_m$ joining the endpoints such that
\begin{equation}
\max_{t\in[0,1]}F_m(\gamma_{l,r}(t))
\le
\Lambda_m(l,r;\lambda)
+\frac{LD_X}{\kappa}r
\bigl(\Lambda_m(l,r;\lambda)-\underline R\bigr),
\label{eq:finite-bound}
\end{equation}
where
\begin{equation}
\Lambda_m(l,r;\lambda)
=\max\!\left\{
\lambda,
e(l),
\delta_{W_A}^{\mathcal L}(m-l,r;m),
\delta_{W_B}^{\mathcal L}(m-l,r;m)
\right\}.
\label{eq:Lambda}
\end{equation}
\end{theorem}

\begin{proof}
We give the construction for the $A$ endpoint and then reverse its analogue for $B$.

First keep $W_A$ fixed and linearly interpolate $\theta_A$ to a minimizer $\beta_A$ of $\delta_{W_A}^{\mathcal L}(m,0;m)$. Both endpoints have objective at most $\lambda$, since $(W_A,\theta_A)$ is feasible in that minimization. Convexity in the logit and of the $\ell_1$ norm keeps the whole segment below $\lambda$.

Let $(\widehat W_A,\eta_A)$ attain $\delta_{W_A}^{\mathcal L}(m-l,r;m)$. Lemma~\ref{lem:bridge} connects $(W_A,\beta_A)$ to $(\widehat W_A,\eta_A)$ below the right-hand side of \eqref{eq:finite-bound}.

Let $(U^\star,\xi^\star)\in\mathcal W_l$ attain $e(l)$. Since $\eta_A$ has at most $m-l$ nonzero coordinates, at least $l$ first-layer rows carry zero output coefficients. Move those rows, at constant predictor and constant objective, to the rows of $U^\star$. With this combined first layer fixed, linearly interpolate the output vector from $\eta_A$ to the corresponding embedding of $\xi^\star$. Convexity bounds this segment by
\[
\max\{\delta_{W_A}^{\mathcal L}(m-l,r;m),e(l)\}\le\Lambda_m(l,r;\lambda).
\]

It remains to ensure that the $A$ and $B$ constructions reach the same parameter point, not merely the same predictor. If $l<m$, there is at least one zero output coordinate after the preceding interpolation. Such a coordinate is a buffer: move its first-layer row to an active row $u$, transfer a coefficient $a$ continuously by replacing $(a,0)$ with $((1-t)a,ta)$ on the two now-identical rows, and then move the newly inactive row freely. Throughout this move the predictor is identical and the penalty equals $|(1-t)a|+|ta|=|a|$. Repeating the move implements any required finite permutation and places the reference atoms in a fixed set of $l$ coordinates and a fixed order. All remaining zero-coefficient rows are then moved to the zero vector.

If $l=m$, then $m-l=0$ forces $\eta_A=\eta_B=0$. On each side all rows may therefore be moved, at constant zero predictor, directly to $U^\star$ in the same prescribed order before interpolating the output coefficients to $\xi^\star$. Thus in both cases the two one-sided constructions reach the identical labeled parameter point, namely the canonical embedding of $(U^\star,\xi^\star)$.

Concatenating the $A$ path with the reverse $B$ path proves the result.
\end{proof}

\begin{remark}[What is new in the bridge]\label{rem:novelty}
The combinatorial path architecture follows the mechanism of Freeman and Bruna \cite{FreemanBruna2017}. The loss-specific step is Lemma~\ref{lem:bridge}: global Lipschitz continuity converts a logit perturbation directly into a first-order risk perturbation. No quadratic expansion is used.
\end{remark}

\begin{lemma}[Monotone sphericalization]\label{lem:sphericalize}
Every $(W,\theta)\in\mathcal W_m$ is connected to a point $(W^\circ,\theta^\circ)\in\mathcal W_m$ such that
\begin{equation}
\theta_i^\circ\ne0
\quad\Longrightarrow\quad
\norm{w_i^\circ}_2=1,
\label{eq:spherical-active}
\end{equation}
the predictor is constant along the path, and the objective is nonincreasing. In particular,
$F_m(W^\circ,\theta^\circ)\le F_m(W,\theta)$ and
$\norm{\theta^\circ}_1\le\norm{\theta}_1$.
\end{lemma}

\begin{proof}
It is enough to treat one row at a time. If $\theta_i=0$, no move is needed. If $\theta_i\ne0$ but $w_i=0$, linearly set $\theta_i$ to zero; the predictor is unchanged because $\relu(0)=0$, and the penalty decreases. Finally, suppose $0<\rho_i:=\norm{w_i}_2<1$. For $t\in[0,1]$ put
\[
c_i(t)=(1-t)+t\rho_i^{-1},
\qquad
w_i(t)=c_i(t)w_i,
\qquad
\theta_i(t)=\theta_i/c_i(t).
\]
Then $1\le c_i(t)\le\rho_i^{-1}$, so $\norm{w_i(t)}_2\le1$. Positive homogeneity gives
$\theta_i(t)\relu(w_i(t)^\top x)=\theta_i\relu(w_i^\top x)$ for every $x$, while $\abs{\theta_i(t)}$ is nonincreasing. The endpoint has first-layer norm one. Concatenating these finitely many rowwise paths proves the claim.
\end{proof}

\section{Barrier Rates and Approximation Transfer}\label{sec:rates}

The preceding lemma uses homogeneity only in the penalty-compatible direction: radii grow to one while output coefficients shrink. It is not a reversible ``without loss of generality'' identification of the ball with the sphere. It does, however, reduce the covering step to the active directional dictionary.

\begin{lemma}[Metric entropy of the active sphere]\label{lem:sphere-cover}
For every $n\ge2$ there is a constant $c_n<\infty$ such that, for every integer $N\ge1$, the sphere $S^{n-1}$ can be covered by at most $N$ sets of Euclidean diameter at most
\begin{equation}
c_nN^{-1/(n-1)}.
\label{eq:sphere-cover}
\end{equation}
Consequently, its Euclidean covering number, defined as the minimum number of
radius-$\varepsilon$ Euclidean balls needed to cover the sphere, satisfies
\[
\mathcal N(S^{n-1},\varepsilon)
\le C_n\varepsilon^{-(n-1)},
\qquad 0<\varepsilon\le2,
\]
for a constant $C_n$ depending only on $n$.
\end{lemma}

\begin{proof}
Split $S^{n-1}$ into the $2n$ coordinate patches on which one prescribed signed coordinate has magnitude at least $1/\sqrt n$. Up to a coordinate permutation, each patch is the graph
$z\mapsto(z,\pm\sqrt{1-\norm{z}_2^2})$ over a subset of $[-1,1]^{n-1}$. On that subset the final coordinate is bounded away from zero, so the graph map has a Lipschitz constant depending only on $n$.

For $N\ge2n$, set $q=\lfloor(N/(2n))^{1/(n-1)}\rfloor$ and divide $[-1,1]^{n-1}$ into $q^{n-1}$ congruent cells. Their intersections with the coordinate patches give at most $2nq^{n-1}\le N$ sets. Their diameters are at most $C_n'/q\le c_nN^{-1/(n-1)}$, where we used $q\ge\tfrac12(N/(2n))^{1/(n-1)}$. Enlarging $c_n$ covers the finitely many cases $N<2n$ by the sphere itself. The covering-number bound follows by choosing $N$ on the order of $\varepsilon^{-(n-1)}$ and placing a Euclidean ball at one point of each nonempty set.
\end{proof}

\begin{lemma}[Spherical cluster merging]\label{lem:merge}
Fix $n\ge2$ and a bounded level $\bar\lambda>\underline R$. There is a constant $C=C(n,L,D_X,\kappa,\bar\lambda-\underline R)$ such that, whenever $F_m(W,\theta)\le\lambda\le\bar\lambda$, every active row has norm one, and $1\le l<m$,
\begin{equation}
\delta_W^{\mathcal L}(m-l,0;m)
\le
\lambda+C\left(\frac{l+1}{m}\right)^{1/(n-1)}.
\label{eq:merge-bound}
\end{equation}
\end{lemma}

\begin{proof}
Let $s=\norm{\theta}_0$. If $s\le m-l$, the original point is already feasible for the compression problem and there is nothing to prove. Otherwise set
\[
k=s-(m-l),
\qquad 1\le k\le l.
\]
Put $N=\lfloor s/(k+1)\rfloor\ge1$ and use the cover from Lemma~\ref{lem:sphere-cover}. Assign each active row to one covering set that contains it. If every set received at most $k$ rows, then
$s\le kN\le ks/(k+1)<s$, a contradiction. Hence some set contains a cluster $Q$ of at least $k+1$ active rows. Since $\lfloor x\rfloor\ge x/2$ for $x\ge1$, its diameter is at most
\[
2^{1/(n-1)}c_n
\left(\frac{k+1}{s}\right)^{1/(n-1)}.
\]
Notice also that
\[
\frac{k+1}{s}
=\frac{s-m+l+1}{s}
\le\frac{l+1}{m},
\]
where the last inequality uses $s\le m$ and $l<m$.

Choose $i_0\in Q$, set
\[
\widetilde\theta_{i_0}=\sum_{i\in Q}\theta_i,
\qquad
\widetilde\theta_i=0\quad(i\in Q\setminus\{i_0\}),
\]
and leave the other coefficients unchanged. Because all coefficients indexed by $Q$ were active, this removes at least $k$ active coefficients and leaves at most $s-k=m-l$ of them. Moreover,
\[
\norm{\widetilde\theta}_1\le\norm{\theta}_1.
\]
Moreover,
\begin{align*}
\E\abs{\Phi_m(X;W,\theta)-\Phi_m(X;W,\widetilde\theta)}
&\le
\sum_{i\in Q\setminus\{i_0\}}\abs{\theta_i}
\E\abs{\relu(w_i^\top X)-\relu(w_{i_0}^\top X)}\\
&\le 2^{1/(n-1)}c_n
\left(\frac{k+1}{s}\right)^{1/(n-1)}D_X\norm{\theta}_1.
\end{align*}
By Lemma~\ref{lem:attainment}, $\norm{\theta}_1\le(\bar\lambda-\underline R)/\kappa$. The Lipschitz property of the loss now gives \eqref{eq:merge-bound}. No perturbation of $W$ is used, which explains the radius $r=0$.
\end{proof}

\begin{proposition}[Explicit ambient-cube certificate]\label{prop:cube-merge}
Without sphericalizing the rows, the same signed merge construction gives
\begin{equation}
\delta_W^{\mathcal L}(m-l,0;m)-F_m(W,\theta)
\le
4\sqrt n\,LD_X\norm{\theta}_1
\left(\frac{l+1}{m}\right)^{1/n}
\label{eq:cube-certificate}
\end{equation}
whenever $1\le l<m$; if the endpoint already has at most $m-l$ active coefficients, the left-hand side is nonpositive.
\end{proposition}

\begin{proof}
When compression is needed, let $s=\norm{\theta}_0$ and $k=s-(m-l)$. Partition $[-1,1]^n$ into $p^n$ congruent cells with
$p=\lfloor(s/(k+1))^{1/n}\rfloor$. The active-row pigeonhole argument gives a cluster of at least $k+1$ rows with diameter at most
$4\sqrt n((k+1)/s)^{1/n}$. Merge its signed coefficients as in Lemma~\ref{lem:merge}; then use
$(k+1)/s\le(l+1)/m$, the ReLU Lipschitz estimate, and loss Lipschitzness. This proves \eqref{eq:cube-certificate}.
\end{proof}

\begin{lemma}[Two-ray compression]\label{lem:two-ray}
If $n=1$, every point of $\mathcal W_m$ is connected by a nonincreasing-objective path to an equivalent representation with at most two active neurons.
\end{lemma}

\begin{proof}
After Lemma~\ref{lem:sphericalize}, active rows belong to $S^0=\{-1,+1\}$. At each direction, replace all coefficients by their signed sum on one representative and set the others to zero. This preserves the predictor and does not increase the $\ell_1$ norm by the triangle inequality. The coefficient transfer is continuous between identical rows, so the objective is nonincreasing throughout. At most two active rows remain.
\end{proof}

\begin{theorem}[Uniform decay above the limiting value]\label{thm:fixed-level}
Let $\lambda>e_\infty$. If $n\ge2$, there are constants $M_\lambda<\infty$ and $C_\lambda<\infty$ such that, for every $m\ge M_\lambda$ and every pair
\[
\vartheta_m^A,\vartheta_m^B\in\Omega_m(\lambda),
\]
there is a continuous path $\gamma_m$ joining them with
\begin{equation}
\max_{t\in[0,1]}F_m(\gamma_m(t))
\le\lambda+C_\lambda m^{-1/(n-1)}.
\label{eq:fixed-rate}
\end{equation}
Equivalently,
\[
\SubBar_m(\lambda)=O(m^{-1/(n-1)}).
\]
The constant is uniform over the endpoints.
If $n=1$, then
\[
\SubBar_m(\lambda)=0
\qquad\text{for every }m\ge4.
\]
\end{theorem}

\begin{proof}
Because $e(l)\downarrow e_\infty$, choose a fixed $l_\lambda$ with $e(l_\lambda)<\lambda$. Concatenate each endpoint with its nonincreasing sphericalization path. For $n\ge2$, apply Lemma~\ref{lem:merge} to the two sphericalized endpoints with $l=l_\lambda$ and $\bar\lambda=\lambda$. Then
\[
\delta_{W_A^\circ}^{\mathcal L}(m-l_\lambda,0;m),
\delta_{W_B^\circ}^{\mathcal L}(m-l_\lambda,0;m)
\le\lambda+C_\lambda m^{-1/(n-1)}.
\]
Use Theorem~\ref{thm:finite} from the sphericalized endpoints with $r=0$. The perturbation term vanishes, while $e(l_\lambda)<\lambda$, and reversing the endpoint sphericalization paths proves \eqref{eq:fixed-rate}. If $n=1$, Lemma~\ref{lem:two-ray} implies $e(2)=e_\infty$: every finite-width network has a no-worse representation with at most two active neurons, while width monotonicity gives the reverse inequality. Continue each sphericalization path with the nonincreasing two-ray compression path and take $l=2$. For every $m\ge4$, each resulting endpoint has at most $2\le m-2$ active coefficients, so both compression values in Theorem~\ref{thm:finite} are at most $\lambda$; moreover $e(2)=e_\infty<\lambda$. The theorem with $r=0$, together with the nonincreasing endpoint paths, gives a path entirely inside $\Omega_m(\lambda)$.
\end{proof}

The fixed-level theorem does not need a quantitative estimate for $e(l)-e_\infty$. Before treating moving levels we derive such an estimate directly for the present regularized model.

\begin{lemma}[Objective-value approximation by sampling]\label{lem:maurey}
There is a constant $A<\infty$, depending only on the standing problem data, such that
\begin{equation}
0\le e(l)-e_\infty\le A l^{-1/2},
\qquad l\ge1.
\label{eq:universal-approx-rate}
\end{equation}
\end{lemma}

\begin{proof}
Let $B_0=F_1(0,0)$, which is finite by Assumption~\ref{ass:standing}, and set $T=(B_0-\underline R+1)/\kappa$. For every $\varepsilon\in(0,1)$ choose a width $q=q(\varepsilon)$ and $(W,\theta)\in\mathcal W_q$ with
$F_q(W,\theta)\le e_\infty+\varepsilon\le B_0+\varepsilon$. Lemma~\ref{lem:attainment} gives $\norm{\theta}_1\le T$. If $\norm{\theta}_1=0$, then both predictors are zero and the width-$l$ zero network satisfies
\[
e(l)\le F_l(0,0)=F_q(W,0)\le e_\infty+\varepsilon;
\]
this is already no larger than the common bound obtained below. Otherwise write $a=\norm{\theta}_1$ and sample independently indices $I_1,\ldots,I_l$ with probabilities $\abs{\theta_i}/a$. The random $l$-neuron predictor
\[
\widehat f_l(x)=\frac{a}{l}\sum_{j=1}^l
\operatorname{sgn}(\theta_{I_j})\relu(w_{I_j}^\top x)
\]
has output $\ell_1$ norm at most $a$ and satisfies $\E_I\widehat f_l(x)=\Phi_q(x;W,\theta)$. Independence and $\relu(w^\top x)^2\le\norm{x}_2^2$ give
\[
\E_I\E_X\abs{\widehat f_l(X)-\Phi_q(X;W,\theta)}^2
\le \frac{a^2}{l}\E\norm{X}_2^2.
\]
By Cauchy--Schwarz and the $L$-Lipschitz property of the loss, the expected objective of the sampled network is at most
\[
F_q(W,\theta)+\frac{La}{\sqrt l}\sqrt{\E\norm{X}_2^2}
\le e_\infty+\varepsilon+
\frac{LT}{\sqrt l}\sqrt{\E\norm{X}_2^2}.
\]
At least one realization obeys the same upper bound. Thus, in either case,
\[
e(l)\le e_\infty+\varepsilon+
\frac{LT}{\sqrt l}\sqrt{\E\norm{X}_2^2}.
\]
Since the construction applies to every $\varepsilon\in(0,1)$, letting $\varepsilon\downarrow0$ proves \eqref{eq:universal-approx-rate}.
\end{proof}

\begin{corollary}[Qualitative moving-level connectivity]\label{cor:qualitative}
Let $\lambda_m$ be any bounded sequence with $\lambda_m\ge e(m)$. Then $\SubBar_m(\lambda_m)\to0$ for $n\ge2$, while for $n=1$,
\[
\SubBar_m(\lambda_m)=0
\qquad\text{for every }m\ge4.
\]
\end{corollary}

\begin{proof}
For $n\ge2$, take any integers $l_m\to\infty$ with $l_m/m\to0$, for example $l_m=\lfloor\sqrt m\rfloor$. Sphericalization, Lemma~\ref{lem:merge}, and Theorem~\ref{thm:finite} give
\[
\SubBar_m(\lambda_m)
\le (e(l_m)-e_\infty)
+C\left(\frac{l_m+1}{m}\right)^{1/(n-1)},
\]
where $C$ is uniform because $(\lambda_m)$ is bounded. Both terms tend to zero. If $n=1$, concatenate each endpoint with the nonincreasing paths from Lemmas~\ref{lem:sphericalize} and~\ref{lem:two-ray}. Every resulting representation has at most two active neurons, and the same lemmas imply $e(l)=e(2)=e_\infty$ for $l\ge2$. Choosing reference width two in Theorem~\ref{thm:finite} then gives zero thickening for every $m\ge4$; reversing the second endpoint path completes the connection.
\end{proof}

\begin{corollary}[Approximation rate implies barrier rate]\label{cor:transfer}
Suppose that for some $A>0$ and $s>0$,
\begin{equation}
0\le e(l)-e_\infty\le A l^{-s},
\qquad l\ge1.
\label{eq:approx-rate}
\end{equation}
Let $\lambda_m$ be any bounded sequence with $\lambda_m\ge e(m)$. Then
\begin{equation}
\SubBar_m(\lambda_m)
=O\!\left(m^{-s/((n-1)s+1)}\right),
\qquad n\ge2.
\label{eq:transfer-rate}
\end{equation}
For $n=1$ the thickening is zero for every $m\ge4$. In particular, these statements control the path thickening between arbitrary global minimizers when $\lambda_m=e(m)$.
\end{corollary}

\begin{proof}
For any $1\le l<m$, first sphericalize both endpoints by Lemma~\ref{lem:sphericalize}. Applying Lemma~\ref{lem:merge} and Theorem~\ref{thm:finite} with $r=0$ to the sphericalized endpoints, and adjoining the nonincreasing endpoint paths, gives uniformly over endpoint pairs in $\Omega_m(\lambda_m)$,
\begin{equation}
\SubBar_m(\lambda_m)
\le
A l^{-s}+C\left(\frac{l+1}{m}\right)^{1/(n-1)},
\label{eq:tradeoff}
\end{equation}
where $C$ is uniform because $(\lambda_m)$ is bounded. Indeed, $\lambda_m\ge e(m)\ge e_\infty$, so the excess of $e(l)$ over $\lambda_m$ is at most $A l^{-s}$. For all sufficiently large $m$, choose the integer
$l_m=\lfloor m^{1/((n-1)s+1)}\rfloor$, so $1\le l_m<m$. The two terms in \eqref{eq:tradeoff} are then both $O(m^{-s/((n-1)s+1)})$; the finitely many remaining widths are absorbed into the constant. This proves \eqref{eq:transfer-rate}. The one-dimensional assertion follows from Corollary~\ref{cor:qualitative}.
\end{proof}

\begin{corollary}[Near-optimal rate under the standing assumptions]\label{cor:near-optimal}
For every bounded sequence $\lambda_m\ge e(m)$,
\[
\SubBar_m(\lambda_m)=O(m^{-1/(n+1)}),
\qquad n\ge2,
\]
and the thickening is zero for every $m\ge4$ when $n=1$.
\end{corollary}

\begin{proof}
Insert the rate $s=1/2$ from Lemma~\ref{lem:maurey} into Corollary~\ref{cor:transfer}.
\end{proof}

\begin{remark}[Two width regimes]
There is no contradiction between Theorem~\ref{thm:fixed-level} and Corollary~\ref{cor:transfer}. A fixed level $\lambda>e_\infty$ permits a fixed $l_\lambda$ and yields the geometric rate $m^{-1/(n-1)}$ for $n\ge2$. Near the limiting optimum, $l$ must grow so that $e(l)$ follows the moving level; balancing approximation and geometric compression produces the slower rate in \eqref{eq:transfer-rate}. Lemma~\ref{lem:maurey} supplies one universal balance, while target-specific information may improve it.
\end{remark}

\section{Theory-Aligned Computational Diagnostic}\label{sec:computational}

\subsection{Design and evidentiary scope}

The numerical experiment instantiates exactly the bias-free model \eqref{eq:model}: every first-layer row is projected onto $B_2^n$, and the $\ell_1$ penalty acts only on the output coefficients. Its regression loss is
\[
\mathcal L_\delta(y,u)=
\begin{cases}\frac12r^2,&\abs{r}\le\delta,\\
\delta(\abs{r}-\frac12\delta),&\abs{r}>\delta,
\end{cases}
\qquad r=u-y,
\]
with $\delta=0.25$; hence it is globally Lipschitz in $u$ with $L=0.25$. The NumPy evaluator and PyTorch optimizer implement this same normalization. For each $n\in\{1,2,4\}$, a seeded finite distribution of 384 points in the unit ball is generated by a width-512 ridge-function teacher. The width grid is $m\in\{8,16,32,64,128\}$, with ten independent pool replicates and eight projected-proximal-Adam solutions per pool.

Within every replicate-width group we select three disjoint endpoint pairs at each of two levels. At a given width, the base level is the largest across replicates of the sixth-smallest objective among the eight trained models. The fixed level is the largest base level across widths within a dimension; the moving level is the width-specific base level plus the prespecified slack $0.02F(0)m^{-1/2}$. This gives 300 pair records per dimension and 900 pair records overall; pairing is disjoint within, but not across, the level groups. The pooled best optimized value $\widehat e(m)$ is only an upper proxy for $e(m)$; neither global optimizer convergence nor sampling of the entire sublevel is assumed. Further deterministic selection details and numerical tolerances are recorded in the supplementary notes.

The two levels represent the two distinct statements proved in Section~\ref{sec:rates}. The common fixed threshold is a finite-sample analogue of the single $\lambda>e_\infty$ in Theorem~\ref{thm:fixed-level}; it isolates the geometric effect of increasing width while keeping the sublevel unchanged. The width-dependent threshold mimics the $\lambda_m$ regime of Corollary~\ref{cor:transfer}, in which the sublevel follows improving approximation values. Because the empirical order statistic is not a certified value of $e(m)$, this second row illustrates the moving-level question but does not verify the hypothesis $\lambda_m\ge e(m)$. Neither choice is required by the path-finding algorithms; both are retained because they diagnose different theoretical regimes.

\begin{lemma}[Segmentwise mesh certificate]\label{lem:dss-certificate}
Consider the empirical objective on inputs $x_1,\ldots,x_N$ and the straight parameter segment
$\gamma(t)=((1-t)W^0+tW^1,(1-t)\theta^0+t\theta^1)$. Put
\begin{align}
K_\gamma={}&\frac{L}{N}\sum_{j=1}^N\left[
\sum_i\abs{\theta_i^1-\theta_i^0}
\max_{q\in\{0,1\}}\relu((w_i^q)^\top x_j)\right.\notag\\
&\left.\hspace{35mm}+
\sum_i\max_{q\in\{0,1\}}\abs{\theta_i^q}
\abs{(w_i^1-w_i^0)^\top x_j}
\right]
+\kappa\norm{\theta^1-\theta^0}_1.
\label{eq:segment-lipschitz}
\end{align}
Then $t\mapsto F_m(\gamma(t))$ is $K_\gamma$-Lipschitz. For $G\ge2$ and an equally spaced grid $t_g=g/(G-1)$,
\begin{equation}
\max_{t\in[0,1]}F_m(\gamma(t))
\le \max_{0\le g<G}F_m(\gamma(t_g))
+\frac{K_\gamma}{2(G-1)}.
\label{eq:mesh-certificate}
\end{equation}
The maximum of \eqref{eq:mesh-certificate} over the final segments certifies the returned piecewise-linear path in exact arithmetic.
\end{lemma}

\begin{proof}
For $s,t\in[0,1]$, add and subtract
$\theta_i(s)\relu(w_i(t)^\top x_j)$. Convexity of ReLU gives
$\relu(w_i(t)^\top x_j)\le\max_q\relu((w_i^q)^\top x_j)$, while
$\abs{\theta_i(s)}\le\max_q\abs{\theta_i^q}$ and ReLU is one-Lipschitz. Summing the resulting predictor bound, applying the $L$-Lipschitz loss, and using
$\abs{\norm{\theta(t)}_1-\norm{\theta(s)}_1}\le\abs{t-s}\norm{\theta^1-\theta^0}_1$
proves the first statement. Every $t$ lies within $1/(2(G-1))$ of a grid point, which proves \eqref{eq:mesh-certificate}.
\end{proof}

The certificate is analytic, but its reported terms are evaluated in double-precision floating-point arithmetic and are therefore understood up to floating-point error; interval arithmetic or outward rounding is not used.

We compare three continuous paths. Direct parameter interpolation is evaluated on 513 points with a global-Lipschitz mesh correction. Our \emph{certified DSS} is a modified version of the Greedy Dynamic String Sampling algorithm introduced by Freeman and Bruna \cite{FreemanBruna2017}. Like their algorithm, it locates a high point on an interpolated segment, optimizes the inserted model down toward the prescribed level, and recursively refines the two resulting segments. We add adaptive grids, an analytic Lipschitz mesh correction, explicit depth/training-limit statuses, and a final maximum over all returned segments. Consequently, ``certified DSS'' below should not be read as an unmodified reproduction of their pseudocode.

The \emph{proof-inspired path} is not DSS. It deterministically executes the representation-changing mechanism of Theorem~\ref{thm:finite}: cluster merging frees coordinates, a common narrow reference network is inserted, and both endpoints are moved to the same canonical embedding. Every segment is either convex output interpolation at fixed first-layer rows or a constant-predictor parameter motion, so its largest node objective is an exact upper bound for the entire continuous path. DSS is instead a numerical path-search heuristic that repeatedly trains high points of parameter-interpolation segments and does not implement the theorem's compression argument. Direct interpolation, certified DSS, and the proof-inspired path therefore answer three different diagnostic questions. The numerical gap reported below is the positive part of the corresponding path upper bound minus the selected level.

The experiment is an observation on a finite seeded distribution and finitely many endpoints. It does not estimate the supremum in $\SubBar_m(\lambda)$ and therefore is not a proof of Theorem~\ref{thm:fixed-level} or Corollary~\ref{cor:transfer}. The independent replicate is the statistical unit; the shaded ranges in Fig.~\ref{fig:path-diagnostics} are minima and maxima of the within-replicate worst values, rather than pair-level confidence intervals.

\subsection{Results}

\begin{figure*}[t]
\centering
\includegraphics[width=\textwidth]{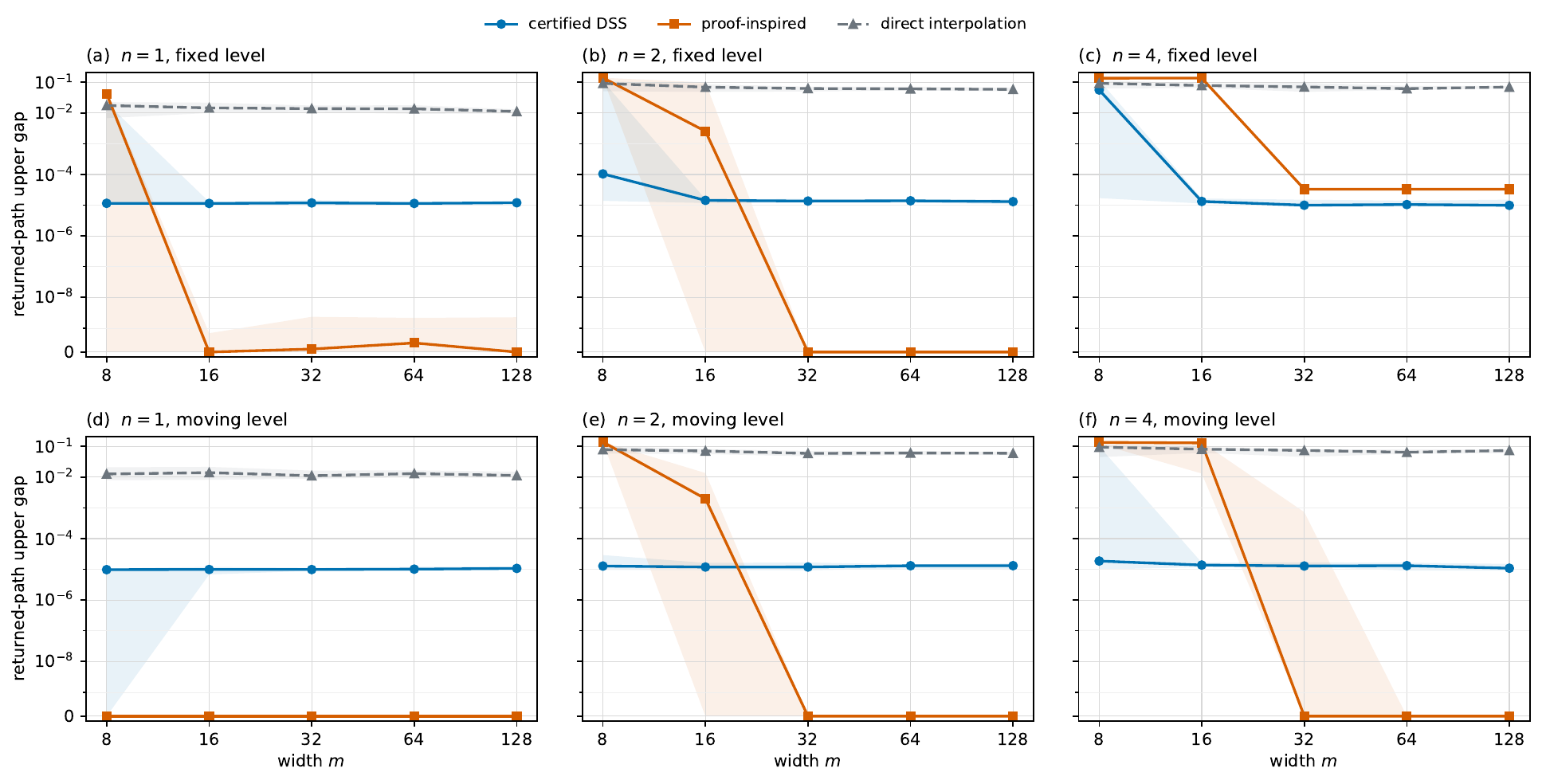}
\caption{Returned-path upper gaps over width. Columns are input dimensions $n=1,2,4$; the top row uses one fixed level across widths and the bottom row uses the width-dependent moving level. Each curve is the median of ten replicate-wise maxima over three disjoint pairs, and its shaded band is the full range of those ten maxima. ``Certified DSS'' is the modified recursive search with a mesh certificate; ``proof-inspired'' is the explicit compression--reference construction; ``direct interpolation'' is the straight parameter segment. Zero is displayed separately on a symmetric-log scale. The abrupt change from $m=8$ to $m=16$, followed by a DSS certificate floor near $10^{-5}$, should not be interpreted as a fitted power law.}
\label{fig:path-diagnostics}
\end{figure*}

Figure~\ref{fig:path-diagnostics} shows a reproducible width transition. At $m=8$, large path gaps occur for some pairs in dimensions two and four. For all 720 pairs with $m\ge16$, however, the smaller of the DSS and constructive upper gaps is at most $1.6584\times10^{-5}$, or at most $1.4391\times10^{-3}$ of the corresponding objective level. The constructive path lies exactly inside the selected level for $90.8\%$, $90.0\%$, and $35.8\%$ of these pairs for $n=1,2,4$, respectively. The one-dimensional saturation control is consistent with its two-ray dictionary: $\widehat e(m)$ changes by less than $2.1\times10^{-13}$ relatively from $m=8$ to $128$.

Direct interpolation remains substantially higher. Over the same 720 pairs its upper gap ranges from $2.70\times10^{-3}$ to $9.55\times10^{-2}$, with median $4.79\times10^{-2}$. The median direct-to-DSS gap ratio is $4.13\times10^3$ overall, while the dimension-wise ratios are reported in Table~\ref{tab:main-computation}. Thus the small returned barriers are a property of path construction, not of closeness along the straight parameter segment.

\begin{table*}[t]
\caption{Frozen main-run summary. All path quantities are upper gaps for the recorded endpoint pairs, not estimates of the uniform sublevel barrier. The three path columns use only $m\ge16$; the compression ratio uses all widths. ``Best path'' means the smaller of the certified DSS and exact constructive upper gaps; $B$ is the analytic cluster-merge increment bound.}
\label{tab:main-computation}
\centering
\scriptsize
\setlength{\tabcolsep}{3pt}
\begin{tabular}{@{}ccccccc@{}}
\toprule
$n$ & $\widehat e(8)$ & $\widehat e(128)$ & zero constructive & max best-path & median direct/DSS & max $\Delta_{\rm cmp}/B$ \\
 & & & gap (\%) & gap & ratio & \\
\midrule
1 & 0.117451 & 0.117451 & 90.8 & 2.338e-09 & 9.94e+02 & 0.089 \\
2 & 0.023544 & 0.023544 & 90.0 & 1.600e-05 & 4.60e+03 & 0.066 \\
4 & 0.010836 & 0.010721 & 35.8 & 1.658e-05 & 6.10e+03 & 0.045 \\
\bottomrule
\end{tabular}
\end{table*}

\begin{figure*}[t]
\centering
\includegraphics[width=\textwidth]{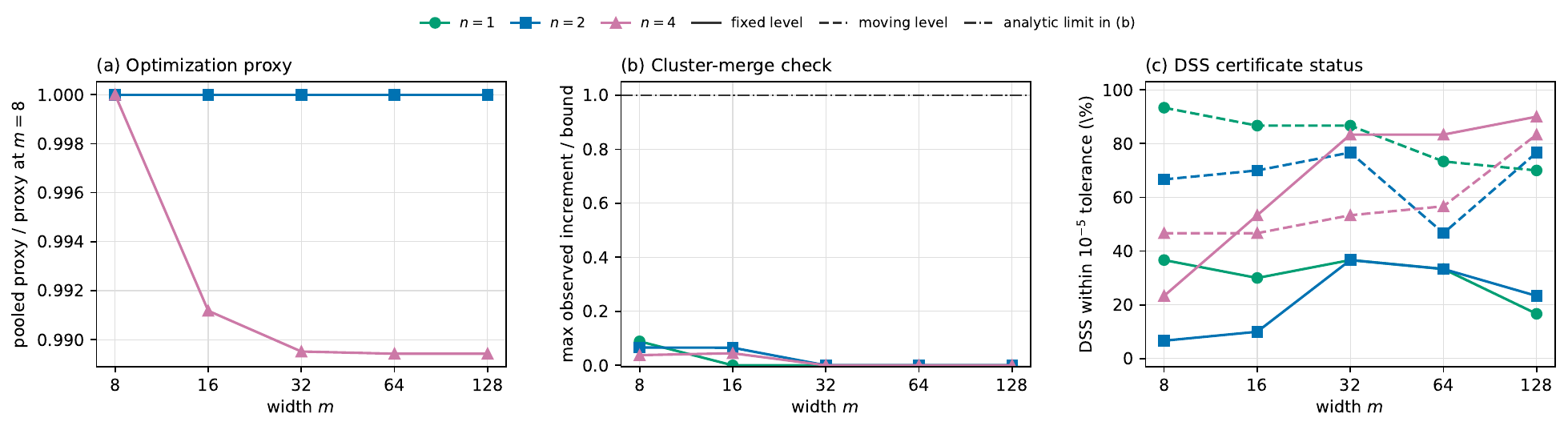}
\caption{Implementation and certificate checks. Color identifies $n$ throughout; in panel (c), solid and dashed curves identify fixed and moving levels. (a) The monotone pooled optimization proxy, normalized by its value at $m=8$; $n=1$ and $n=2$ saturate at the tested resolution, while $n=4$ decreases by $1.06\%$. (b) For each dimension and width, the largest positive cluster-merge increment over 80 trained endpoints divided by the explicit ambient-cube bound in Proposition~\ref{prop:cube-merge}; the dash-dotted horizontal line is the analytic limit one. (c) Fraction of the 30 certified-DSS records at each dimension, width, and level whose upper path value lies within numerical tolerance $10^{-5}$ of the selected level. Legends are placed above the panels to avoid obscuring curves.}
\label{fig:implementation-checks}
\end{figure*}

Figure~\ref{fig:implementation-checks} separates three checks that should not be conflated with the path gaps in Fig.~\ref{fig:path-diagnostics}. All 1200 trained endpoint records satisfy the first-layer constraint, with largest stored row norm $1+2.2\times10^{-16}$. All 1200 compression records satisfy the analytic increment bound. The largest observed-to-analytic increment ratios are $0.089$, $0.066$, and $0.045$ for dimensions one, two, and four. These ratios show that the worst-case covering estimate is conservative on the realized configurations and provide an implementation-level check of the analytic inequality. At $m\ge32$, every trained endpoint already has active support at most $m-4$, so the prescribed four-coordinate compression is a no-op and its increment is zero. These zeros diagnose optimizer-induced sparsity; they are not an empirical estimate of the covering exponent.

To exercise cluster merging rather than inherit this proximal sparsity, we added a dense-representation stress test on the six fixed-level endpoints selected in every replicate--width group with $m\ge16$. Each endpoint is first sphericalized exactly and every active coefficient is split among duplicate copies until all $m$ coordinates are active. In dimensions two and four the copies are then given a deterministic tangent perturbation, choosing the largest prespecified amplitude that remains within the recorded level up to its $10^{-5}$ tolerance. A deterministic nearest-cluster rule merges enough coefficients to free four coordinates. The construction is representation-level: it controls the tested density while preserving the architecture and starting from the same trained functions, but it is not a new sample of optimizer outputs.

\begin{figure*}[t]
\centering
\includegraphics[width=\textwidth]{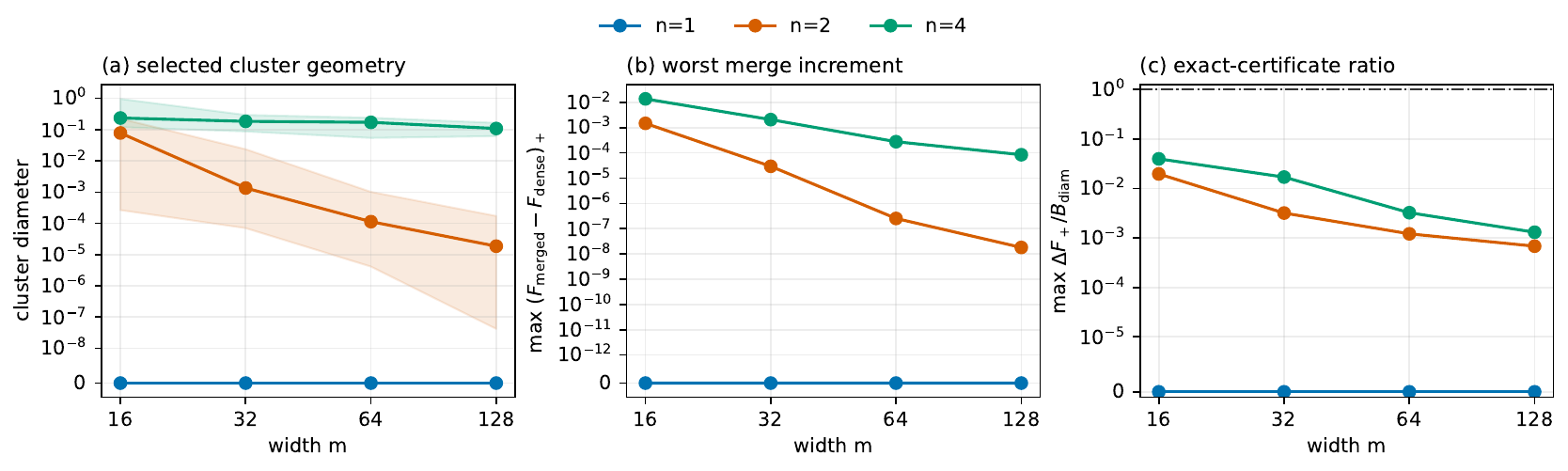}
\caption{Dense-endpoint cluster-merging stress test on 720 fixed-level Huber endpoint representations at widths $m\ge16$. Before merging, all coordinates are active; at least four are removed in every record. (a) Median selected cluster diameter with the full range across 60 endpoints at each dimension and width. (b) Largest positive objective increment. (c) Largest ratio to the individual certificate $B_{\rm diam}=LD_X\norm{\theta}_1\operatorname{diam}(Q)$; the dash-dotted line marks one. The zero one-dimensional curves reflect the exact two-ray dictionary. This constructed diagnostic checks nontrivial merging and its analytic inequality, not an asymptotic exponent.}
\label{fig:dense-stress}
\end{figure*}

Figure~\ref{fig:dense-stress} removes at least four active coefficients in every one of the 720 dense records. No individual diameter certificate is violated. The largest positive merge increment is $1.3668\times10^{-2}$, attained at $n=4,m=16$, and the largest ratio to its certificate is $0.0396$. At $n=2$, the median selected diameter falls from $7.81\times10^{-2}$ at $m=16$ to $1.86\times10^{-5}$ at $m=128$; at $n=4$ it falls from $0.236$ to $0.109$. The finite construction and nearest-cluster selection differ from the worst-case spherical covering proof, so these decreases are reported only as mechanism checks, not as estimates of the rate in Theorem~\ref{thm:fixed-level}.

Corrected DSS returns a finite path and an analytically mesh-corrected upper value even when recursion reaches its depth limit. For $m\ge16$, $56.5\%$ of its pair records are certified within the designated $10^{-5}$ tolerance; the remaining $43.5\%$ contain at least one unresolved final segment, but their reported gaps already include the segmentwise Lipschitz correction. The median certified DSS gap is $9.94\times10^{-6}$. This tolerance-scale plateau, together with the exact zeros of the constructive path, makes a log--log slope unstable: the finite-range fits are dominated by the $m=8$ to $m=16$ transition. We therefore do not use fitted exponents as evidence for the asymptotic rates.

\subsection{Loss robustness: binary cross-entropy}

To test whether the diagnostic is specific to the Huber regression objective, we repeated the complete 900-pair protocol with deterministic labels $y=\mathbf 1\{z_{\rm teacher}\ge0\}$ and the binary cross-entropy loss in
\eqref{eq:bce-logit}. The architecture, inputs, teacher realizations, regularization, optimization budgets, reference-width candidates, levels, pair counts, and path-certification budgets were otherwise unchanged. The positive-class fractions were $0.505$, $0.646$, and $0.583$ for $n=1,2,4$. This rerun directly exercises a second loss covered by Assumption~\ref{ass:standing}, since its global logit Lipschitz constant is one. All 1200 BCE endpoint models remain in the first-layer unit ball, all 1200 compression records satisfy the analytic bound, and all 300 groups contain three disjoint pairs whose endpoints lie in their recorded sublevel.

Raw Huber and cross-entropy objectives have different scales, so Fig.~\ref{fig:loss-robustness} reports the better of the certified DSS and constructive upper gaps divided by the selected level. Fixed and moving levels are shown separately. Each point is the median of ten replicate-wise maxima over three disjoint pairs, and each band is the full range across those replicates; thus the plot does not treat the pair rows as independent observations.

The fixed-level BCE row is qualitatively consistent with a width transition, but is less sharply resolved than Huber: over its 360 pairs at $m\ge16$, the maximal best upper gap is $7.13\times10^{-3}$, or $5.03\times10^{-2}$ of the selected level. The moving-level BCE row is more demanding, with maximal values $1.02\times10^{-1}$ and $7.96\times10^{-1}$, respectively; the largest case occurs at $n=4,m=64$. The proof-inspired path is exactly inside the recorded level for $51.7\%$ of fixed and $57.5\%$ of moving BCE pairs, while DSS is within its $10^{-5}$ tolerance for $31.4\%$ and $23.3\%$.

\begin{figure*}[b]
\centering
\includegraphics[width=\textwidth]{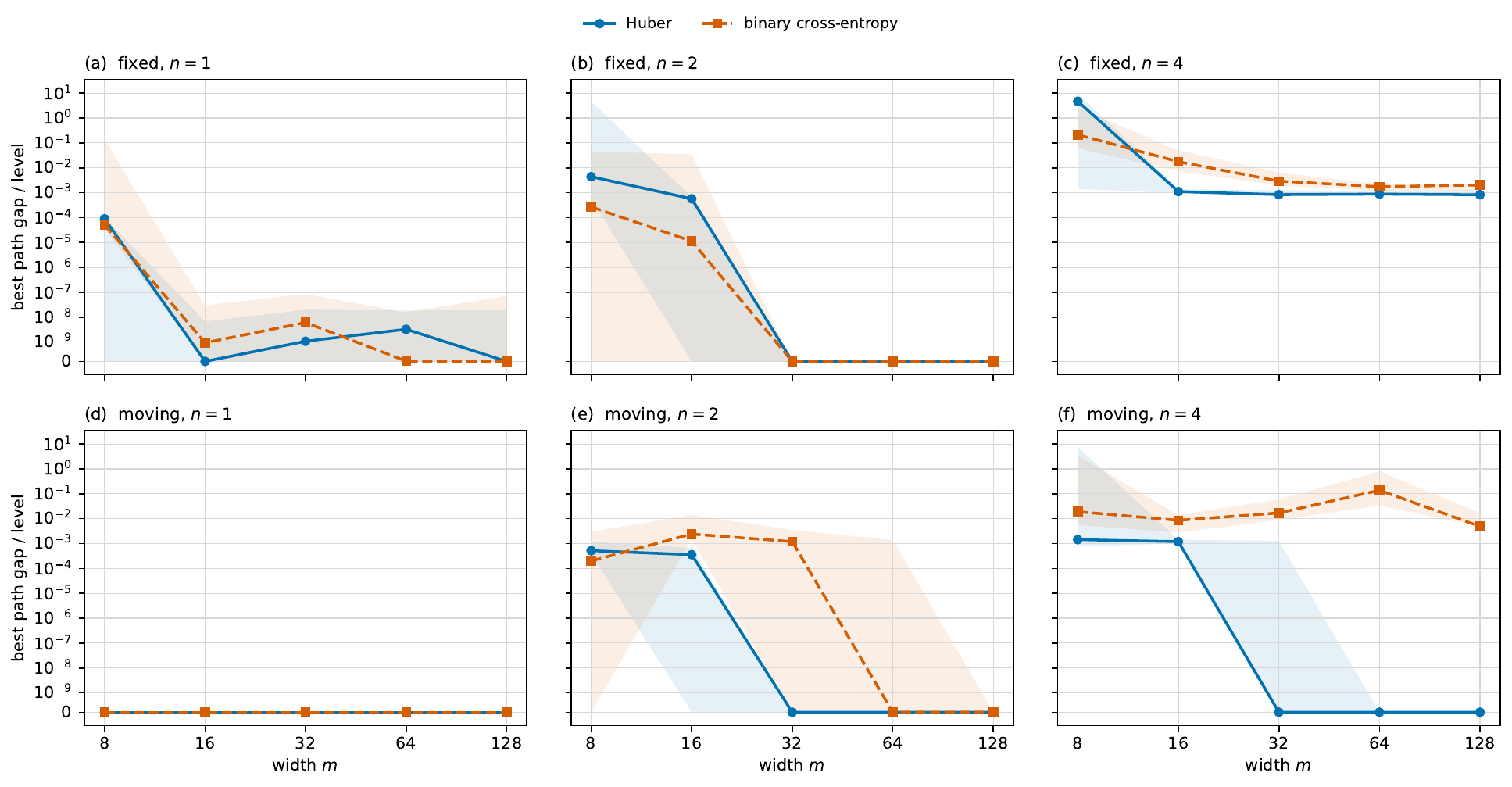}
\caption{Loss-robustness diagnostic. Columns are input dimensions and rows separate the fixed and moving empirical sublevels. Curves show the smaller of the certified DSS and exact constructive upper gaps, normalized by the selected level. Each marker is the median of ten replicate-wise worst values over three disjoint pairs; shading is the full replicate range. The symmetric-log scale displays exact zeros without replacing them by an artificial plotting floor. The comparison is descriptive: different losses induce different objective geometries, and the moving empirical level is not a certified value of $e(m)$.}
\label{fig:loss-robustness}
\end{figure*}

\begin{table*}[t]
\caption{Loss-robustness summary over all dimensions at widths $m\ge16$. Each loss--level row contains 360 recorded pairs. ``Best'' is the smaller certified DSS or constructive upper gap.}
\label{tab:loss-robustness}
\centering
\scriptsize
\setlength{\tabcolsep}{2.2pt}
\begin{tabular}{@{}llccccc@{}}
\toprule
Loss & level & max gap & max/level & proof zero (\%) & DSS tol. (\%) & med. direct/level \\
\midrule
Huber & fixed & 1.649e-05 & 1.377e-03 & 55.8 & 44.2 & 2.246e+00 \\
Huber & moving & 1.658e-05 & 1.439e-03 & 88.6 & 68.9 & 2.138e+00 \\
BCE (logits) & fixed & 7.134e-03 & 5.032e-02 & 51.7 & 31.4 & 2.317e-01 \\
BCE (logits) & moving & 1.016e-01 & 7.956e-01 & 57.5 & 23.3 & 5.991e-01 \\
\bottomrule
\end{tabular}
\end{table*}

This weaker moving-level outcome is informative rather than contradictory. Corollaries~\ref{cor:qualitative}--\ref{cor:near-optimal} assume a certified sequence $\lambda_m\ge e(m)$ and balance a growing reference width against compression. Here $\lambda_m$ is only an optimizer-based proxy, the endpoints are a finite sample, and the numerical algorithms have fixed depth and training budgets. Moreover, a returned-path upper gap is not a lower bound on the infimal barrier: a large value can diagnose failure of the searched constructions without establishing that every path has a large barrier. Consequently, the BCE rerun supports the loss-side generality and implementation invariants, but it is deliberately not presented as empirical verification of the moving-level rate. This is the practical reason for retaining both rows: the fixed row is the closer diagnostic for Theorem~\ref{thm:fixed-level}, whereas the moving row exposes the additional approximation and optimization demands of the moving-level corollaries.

\section{Discussion and Limitations}\label{sec:discussion}

The main theorem is a statement about constrained one-hidden-layer ReLU networks, expected regularized risk, and scalar logits. The first-layer constraint is essential. With an unrestricted first layer and a penalty only on the output layer, positive homogeneity permits the same predictor to be represented with arbitrarily small regularization cost.

The global Lipschitz hypothesis is also essential to the stated bridge estimate. It includes common binary classification losses in logits but excludes mean squared error on an unbounded logit domain. Locally Lipschitz or polynomial-growth losses may admit variants under additional moment and pathwise-logit bounds; those extensions are not asserted here.

The rate $m^{-1/(n-1)}$ comes from a worst-case covering of the spherical active-direction dictionary. It can still be pessimistic when the active rows occupy a set of smaller metric dimension. A data- or dictionary-dependent entropy may sharpen the result further, but would require a uniform structural hypothesis on every endpoint in the sublevel.

Lemma~\ref{lem:maurey} supplies the universal regularized-value rate $O(l^{-1/2})$ without a separate target-class assumption. Corollary~\ref{cor:transfer} remains useful because sharper function-approximation results may improve the landscape exponent, but importing them requires matching the bias-free dictionary, data distribution, loss, and regularization rather than quoting an incompatible approximation norm.

Finally, barrier decay does not by itself imply rapid optimization. The proof constructs an admissible path but gives neither a conditioning estimate nor a guarantee that gradient-based training finds that path.

\section{Conclusion}\label{sec:conclusion}

For a constrained shallow ReLU model, convexity in the output layer and Lipschitz continuity of the loss are sufficient to turn robust neuron compression into a quantitative pathwise barrier bound. A monotone sphericalization path and cluster merging give uniform $O(m^{-1/(n-1)})$ thickening of every fixed sublevel above the limiting approximation value for $n\ge2$, with exact connectivity for every $m\ge4$ in one dimension. Objective-value sampling proves $e(l)-e_\infty=O(l^{-1/2})$ inside the model and, under the standing assumptions, yields the near-optimal barrier rate $O(m^{-1/(n+1)})$. The resulting statement is a concrete bridge from approximation theory to loss-landscape topology, with its architectural and loss assumptions explicit.

\backmatter

\section*{Acknowledgements}
The author is grateful to Prof. Dmitry A. Kamaev for valuable discussions and guidance.

\section*{Declarations}

\textbf{Funding.} No funding was received for conducting this study.

\textbf{Competing Interests.} The author declares no competing interests.

\textbf{Ethics Approval.} Not applicable.

\textbf{Consent to Participate.} Not applicable.

\textbf{Consent for Publication.} Not applicable.

\textbf{Data Availability.} No external data are used. The synthetic empirical distributions and frozen endpoint and path arrays underlying the reported numerical results are available in the versioned repository snapshot \href{https://github.com/Safreliy/energy-gap-estimate/tree/d645717a7f82658d913703700549e609c79b0b25}{commit \texttt{d645717}}.

\textbf{Code Availability.} The theory-aligned experiment, invariant tests, versioned configurations, proof-audit notes, frozen result tables, and endpoint/path array archives are available at \url{https://github.com/Safreliy/energy-gap-estimate}. The corrected historical notebooks and standalone DSS implementation are isolated under \texttt{legacy/}; they are not used to generate the reported results. The repository contains the frozen canonical artifacts used to regenerate every reported table and figure.

\textbf{Author Contribution.} Saveliy Baturin conducted the analysis, implemented the computational diagnostic, and wrote the manuscript.

\bibliography{bibliography}

\end{document}